\documentclass[acmsmall]{acmart}

\usepackage{lineno}   
\usepackage{booktabs} 
\usepackage{url}  

\usepackage{multirow}  
\usepackage{color}  
\newtheorem{strategy}{Strategy}    
\newtheorem{property}{Property}    
\newtheorem{upper bound}{Upper bound}

\usepackage{diagbox}
\usepackage{amsmath}
\usepackage{arydshln} 
\usepackage{color} 

\usepackage[linesnumbered,ruled]{algorithm2e} 

\SetAlFnt{\small}
\SetAlCapFnt{\small}
\SetAlCapNameFnt{\small}
\SetAlCapHSkip{0pt}
\IncMargin{-\parindent}

\usepackage{xspace}
\usepackage{multirow}
\usepackage{subfigure}

\setcopyright{rightsretained}
\acmJournal{TDS}

\begin{document}
\title{Towards Revenue Maximization with Popular and Profitable Products}

\author{Wensheng Gan}
\affiliation{%
	\institution{Jinan University}
	\city{Guangzhou}
	\country{China}
}
\email{wsgan001@gmail.com}

\author{Guoting Chen}
\affiliation{%
	\institution{Harbin Institute of Technology (Shenzhen)}
	\city{Shenzhen}
	\country{China}	
}
\email{chenguoting@hit.edu.cn}

\author{Hongzhi Yin}
\affiliation{%
	\institution{The University of Queensland}
	\city{Queensland}
	\country{Australia}
}
\email{h.yin1@uq.edu.au}

\author{Philippe Fournier-Viger}
\affiliation{%
	\institution{Shenzhen University}
	\city{Shenzhen}
	\country{China}	
}
\email{philfv@szu.edu.cn}

\author{Chien-Ming Chen}
\authornote{This is the corresponding author}
\affiliation{%
	\institution{Shandong University of Science and Technology}
	\city{Qingdao}
	\country{China}	
}
\email{chienmingchen@ieee.org}

\author{Philip S. Yu}
\affiliation{%
	\institution{University of Illinois at Chicago}
	\city{Chicago}
	\country{USA}
}
\email{psyu@uic.edu}


\begin{abstract}
	Economic-wise, a common goal for companies conducting marketing is to maximize the return revenue/profit by utilizing the various effective marketing strategies. Consumer behavior is crucially important in economy and targeted marketing, in which behavioral economics can provide valuable insights to identify the biases and profit from customers. Finding credible and reliable information on products' profitability is, however, quite difficult since most products tends to peak at certain times w.r.t. seasonal sales cycle in a year. On-Shelf Availability (OSA) plays a key factor for performance evaluation. Besides, staying ahead of hot product trends means we can increase marketing efforts without selling out the inventory. To fulfill this gap, in this paper, we first propose a general profit-oriented framework to address the problem of revenue maximization based on economic behavior, and compute the \underline{O}n-shelf \underline{P}opular and most \underline{P}rofitable \underline{P}roducts (OPPPs) for the targeted marketing. To tackle the revenue maximization problem, we model the $k$-satisfiable product concept and propose an algorithmic framework for searching OPPP and its variants. Extensive experiments are conducted on several real-world datasets to evaluate the effectiveness and efficiency of the proposed algorithm.	
\end{abstract}


\authorsaddresses{\textbf{Authors' addresses}: 
Wensheng Gan, Jinan University, Guangzhou, China, wsgan001@gmail.com; Guoting Chen, Harbin Institute of Technology (Shenzhen), Shenzhen, China, chenguoting@hit.edu.cn; Hongzhi Yin, The University of Queensland, Queensland, Australia, h.yin1@uq.edu.au; Philippe Fournier-Viger, Shenzhen University, Shenzhen, China, philfv@szu.edu.cn; Chien-Ming Chen, Shandong University of Science and Technology, Qingdao, China, chienmingchen@ieee.org; Philip S. Yu, University of Illinois at Chicago, Chicago, USA, psyu@uic.edu
}


\renewcommand\shortauthors{W. Gan et al.}

\maketitle

\section{Introduction}

Economic-wise, a common goal for companies conducting marketing is to maximize the return profit by utilizing various effective marketing strategies. Consumer behavior plays a very important role in economy and targeted marketing \cite{peng2012finding,xu2016product,yang2016influence,zhang2018explainable}. A successful business influences the behavior of consumers to encourage them buying its products. In marketing, behavioral economics can provide valuable insights via various web services, e.g., Amazon, by helping people to identify the biases and profit from all customers. On the other hand, manufacturers can use the information of consumers' requirements on various products to select appropriate products in the market. As a result, evolving the ecosystem of personal behavioral data from web services has many real applications \cite{agrawal1994fast,geng2006interestingness,han2004mining,quattrone2016benefits}. However, understanding consumer behavior is quite challenging, such as finding credible and reliable information on products' profitability. 

The consumer behavior analytics is crucially important for decision maker, which can be used to support the global market and has attracted many considerable attentions \cite{ahmed2009efficient,tseng2013efficient,peng2012finding,xu2016product,yang2016influence}. In economics, \textit{utility} \cite{marshall2009principles} is a measure of a consumer's preferences over alternative sets of goods or services. Specifically, only few works of data mining and information search have been studied \cite{teng2015effective,quattrone2016benefits,xu2016product,zhang2016economic,zhang2018explainable}, from a economic perspective, to study targeted marketing revenue/profit maximization based on economic behavior.  Consider the consumer choice \cite{coleman1992rational} and the preferences constraint, several crucial factors for the task of commerce has been succeeded (w.r.t. profit/utility maximization): i) rational behavior; ii) preferences are known and measurable; iii) inventory management; and iv) price. To obtain higher profit from the products, the decision-maker should find out the most profitable and popular products from the historical records. The reason is that the business people would likely know the amount of profit in their business, but it is, difficult to know which product makes the most money \cite{peng2012finding,xu2016product}.

In recent decades, there are many studies have been proposed for computational economics which explores the intersection of economics and computation. In the same time, the scientists in field of computer science also incorporate some interesting concepts (e.g., utility theory) form Economics into various research domains, such as data mining, database, social network, Internet of Things, etc. In the field of data mining which also called knowledge discovery form data (KDD), there is tremendous interest in developing novel utility-oriented methodologies and models for obtaining insights over rich data. Therefore, a new utility-oriented data mining paradigm called utility mining \cite{tseng2013efficient,gan2018survey,gan2020survey} becomes an emerging technology and successfully be applied to various fields. For example, high-utility itemset mining \cite{tseng2013efficient,mai2017lattice,nguyen2019mining}, high-utility sequence mining \cite{yin2012uspan,lan2014applying,wang2018incremental}, and high-utility episode mining \cite{wu2013mining,lin2015discovering} been extensively studied to deal with different types of data, including itemset-based transactional data, sequence data, and complex event data.

Most products operate on a seasonal sales cycle that tends to peak at a certain time in a year. How to sell the seasonal products to make profit maximization, the sale department can find what products are most likely purchased on special periods, by utilizing different seasonal consumer behavior data (i.e., weekly sales data, monthly sales data, quarterly sales data, and annual sales data). Despite On-Shelf Availability (OSA) is a key factor \cite{corsten2003desperately}, out of stock levels on the shelf, as shown in Fig. 1\footnote{\url{https://www.groceryinsight.com/blog/}}, still remains highly persistent, and today's economic environment is more challenging since it is even more critical than ever for retailers and manufacturers to ensure that every product a customer wants to buy is available every time. The historical data was analyzed to classify the groups of products and identify those are most likely to be sold-out first. A sustainable OSA management process should consider the frequency and quantity of ordering, as well as the inventory, to address the root causes of out-of-stocks. In general, the seasonal products may become popular/hot on some special periods. The popular/hot products that are flying off the shelves. For example, sunglasses may sold out during the summer, but for the coming up winter, fewer people would purchase it. 

\begin{figure}[htbp]
\centering
\includegraphics[scale=0.10]{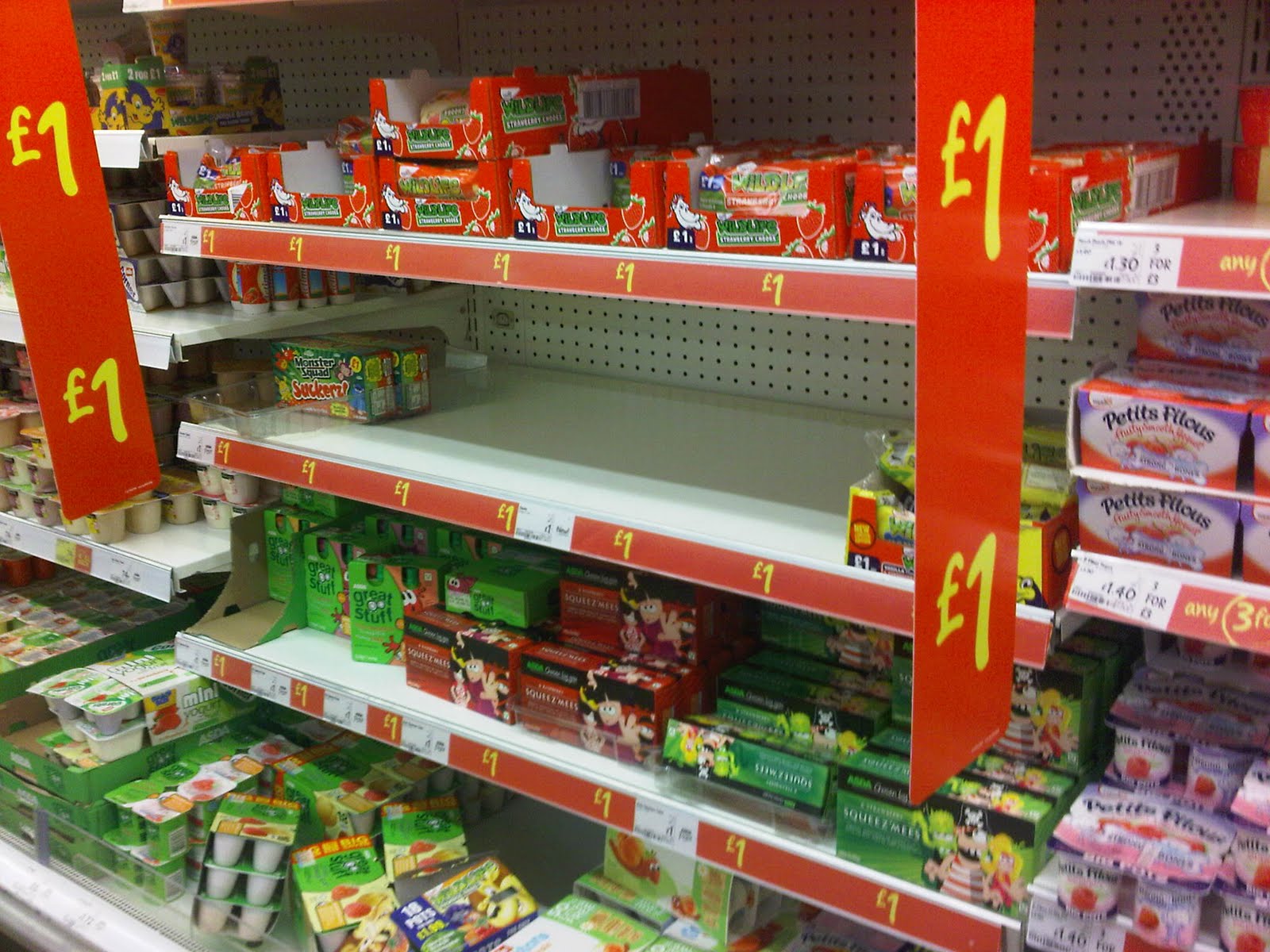}
\captionsetup{justification=centering}
\caption{Out of stock products.}
\label{fig:OSA_Product}	
\end{figure}

Product bundling is one of the sales strategy to make the successful commerce. Use product bundling to pair less profitable (even has negative profit) or slower-moving products with more profitable ones to reduce the storage savings of the less valuable products, which making more space to keep the products with higher profitability. To improve the overall business' profitability, it is important to decide when and how to adjust product pricing and sell bundled products for further increasing profitability. Compared to the hot-sold out products, finding the most profitable products can increase the overall profitability \cite{ge2015dominance}. Ensuring most hot and profitable products on the shelf is essential for any retailer, but even today, it still remains a major challenge.

Motivated by existing research in economics, in this paper, we propose a profit-oriented framework to address the problem and compute the on-shelf popular and most profitable products for the targeted marketing. We also model $k$-satisfiable product searching and propose an algorithmic framework. The principle contributions of this paper are summarized as follows:

\begin{itemize}
\item This is the first work to systematically study the problem of computing on-shelf hot and most profitable products for the targeted marketing  based on economic behavior, including purchase frequency, purchase time periodic, on-shelf availability and utility/profit theory. It can help us to make understand users' economic behaviors, find out the \underline{O}n-shelf \underline{P}opular and most \underline{P}rofitable \underline{P}roducts (OPPPs), then make targeted marketing with profit maximization.	 

\item The solution space of the designed approach can be reduced to a finite the number of points by tree-based searching technique. Two compact data structures  are developed to store the necessary information of the databases. 

\item The concept of remaining positive profit is adopted to calculate the estimated upper bound. Based on developed pruning strategies, the OP3M algorithm can directly discover OPPPs using OPP-list with only twice database scans. Without the candidate generation-and-test, it performs a depth-first search by spanning the search space during the  constructing process of the OPP-list. 

\item The extensive performance evaluation on several real-world e-commerce datasets demonstrates the effectiveness and efficiency of the proposed OP3M framework.
\end{itemize}

The rest of this paper is organized as follows. Some related works are reviewed in Section \ref{sec:relatedwork}. The preliminaries and problem statement are given in Section \ref{sec:preliminaries}. Details of the proposed OP3M algorithm are described in Section \ref{sec:technic}. The evaluation of the effectiveness and efficiency of the proposed OP3M framework are provided in Section \ref{sec:experiments}.  Finally, some conclusions are drawn in Section \ref{sec:conclusion}.

\section{Related Work}  \label{sec:relatedwork}

Our research is related to the work in computational economic, utility-based mining, other profit-oriented searching and mining works. In particular, the advent of Internet has resulted in large sets of user behavior records, which makes it possible for targeted marketing to maximize the return profit. Up to now, evolving the ecosystem of personal behavioral data from web services has many real applications \cite{agrawal1994fast,geng2006interestingness,han2004mining,quattrone2016benefits,fournier2017survey}. For example, frequency-based pattern or rule mining \cite{agrawal1994fast,han2004mining} is one of the common  approaches to discover hidden relationships among items in the transaction. Different from the frequency-based mining model, the rare pattern mining framework that aim at discovering the non-frequent but interesting patterns also has been proposed \cite{koh2016unsupervised}. Geng and Hamilton \cite{geng2006interestingness} reviewed some measures that are intended for selecting and ranking patterns according to the potential interest to the user. The consumer behavior analytics is crucially important for decision-maker, which can be used to support the global market and has attracted many considerable attentions \cite{tseng2013efficient,peng2012finding,xu2016product,yang2016influence}. Recently, some researchers studied the profit of web mining such as profit-oriented pattern mining \cite{ahmed2009efficient,tseng2013efficient}, association \cite{yang2007towards}, market share \cite{quattrone2016benefits}, and decision making to maximize the revenue from products \cite{agrawal1994fast,geng2006interestingness,teng2015effective}. Periodicity is prevalent in physical world, and many events involve more than one periods. The previous profit-oriented works \cite{liu2005two,tseng2013efficient} and skyline operator \cite{ugarte2017skypattern} have not yet optimized to handle temporal on-shelf data, on-shelf availability \cite{corsten2003desperately}, even considering both positive and negative unit profits. In this class of data regarding time, the period of interest needs to be added as an additional constraint to be evaluated together with the decision criteria of the addressed problem.

In economics, \textit{utility} \cite{marshall2005principles} is a key measure of a consumer's preferences over alternative sets of goods or services. It is a basic building block of rational choice theory \cite{coleman1992rational}. Up to now, some works that exploring the economic behavior data have been studied \cite{teng2015effective,quattrone2016benefits,xu2016product,zhang2016economic}, from the data mining and information search perspectives. For example, considering the utility concept, a new data mining paradigm called utility mining \cite{tseng2013efficient,gan2018survey,gan2020survey,mai2017lattice,wang2018incremental} has been extensively studied and applied to different applications. Although a straightforward enumeration of all high-utility patterns (HUPs) sounds promising, it unfortunately does not yield a scalable solution for utility computation of patterns. The reason is that utility in HUPs does not hold the well-known Apriori property \cite{agrawal1994fast} (aka the downward closure property \cite{agrawal1994fast}) \cite{ahmed2009efficient,tseng2013efficient,gan2018survey,gan2020survey}. Utility mining has been extended to deal with different types of data, including itemset-based transactional data \cite{tseng2013efficient,mai2017lattice}, sequence data \cite{yin2012uspan,lan2014applying,wang2018incremental}, uncertain data \cite{2lin2016efficient,lin2017efficiently}, and complex event data \cite{wu2013mining,lin2015discovering}. Furthermore, various interesting and challenging issues about utility mining have been addressed, including top-$K$ high-utility pattern mining \cite{tseng2016efficient,duong2016efficient,yin2013efficiently}, utility mining in dynamic databases \cite{lin2015incremental,lin2016fast,gan2018survey}, mining high-utility patterns by taking different special constraints into account \cite{lin2015efficient,1lin2016fast,lan2015fuzzy,nguyen2019mining}, and privacy preserving utility mining \cite{gan2018privacy}, etc. Overall, utility-driven pattern mining has been shown to be of considerable value in a wide range of applications.

There are many heuristic search algorithms in artificial intelligence such. Up to now, there has been several studies about heuristic search \cite{guns2011declarative}, constraint programming \cite{guns2011itemset,guns2011k}, and multi-objective optimization  \cite{ugarte2017skypattern} such as Pareto for pattern mining. Notice that exhaustive search strategy explores many possible subsets, while heuristic search strategy explores a limited number of possible subsets. Thus, they are different. In particular, in heuristic search methods, the state space is not fully explored and randomization is often employed. Most of the studies of utility mining and pattern mining aim at discovering an optimal set of interesting patterns under the given constraints, while some approaches may lead to not necessarily optimal result by the heuristic search.

In other related research fields, some interesting  works have applied the utility theory \cite{marshall2005principles} to recommender systems \cite{li2011towards,wang2011utilizing,ying2014semantic,zhang2018explainable}. Wang and Zhang \cite{wang2011utilizing} first incorporates marginal utility into product recommender systems. They adapt the widely used Cobb-Douglas utility function \cite{coleman1992rational} to model product-specific diminishing marginal return and user-specific basic utility to personalize recommendation. Li et al. \cite{li2011towards} highlights that product recommender systems differ from the music or movie recommender systems as the former should take into account the utility of products in their ranking. It employs the utility and utility surplus \cite{zhang2018explainable} theories from economics and marketing to improve the list of recommended product. \cite{zhao2017multi} finds multi-product utility maximization for economic recommendation. These existing product recommender systems, however, do not consider the time periodicity between the products purchased and on-shelf availability. The work in \cite{zhao2012increasing} utilizes the purchase interval information to improve the performance for e-commerce, but it does not consider the on-shelf availability and purchase frequency of products.

\section{Preliminaries and Problem Formulation}
\label{sec:preliminaries}

\subsection{Utility-based Computing Model}

A fundamental notion in \textit{utility theory} is that each consumer is endowed with an associated \textit{utility function}, which is ``a measure of the satisfaction from consumption of various goods and services" \cite{marshall2009principles}. In the context of purchasing decisions, we assume that the consumer has access to a set of products, each product having a price. Informally, buying a product involves the exchange of money for a product. Given the utility of a product, to analyze consumers' motivation to trade money for the product, it is also necessary to analyze consumer behavior. In economics, the utility that a consumer has for a product can be decomposed into a set of utilities for each product characteristic. According to this utility theory, we have the following concepts and formulation. The notations of symbols are first summarized in Table \ref{table_Notation}.

\begin{table}[!htbp]
\centering
\small
\caption{Summary of Notations}
\label{table_Notation}
\begin{tabular}{|c|l|}
	\hline
	\textbf{Symbol} & \textbf{Description}  \\ \hline
	$I$ &  A set of $m$ items/products, \textit{I} = \{\textit{i}$_{1}$, \textit{i}$_{2}$, $\ldots$, \textit{i$_{m}$}\}. \\ \hline
	
	$X$ & A group of products $X$ = $\{i_1,$ $i_2,$ $ \dots, i_j\}$. \\ \hline
	
	$D$	&  A quantitative database, \textit{D} = \{\textit{T}$_{1}$, \textit{T}$_{2}$, $\ldots$, \textit{T$_{n}$}\}.  \\ \hline
	
	\textit{minfre} &   A  minimum frequent threshold. \\ \hline
	
	\textit{minpro} &   A  minimum profit threshold. \\ \hline

	$sup(X)$ &   The total support value of $X$ in $D$. \\ \hline
	
	$\textit{OPPP}$ &   On-shelf most popular and profitable product. \\ \hline
	
	$q(i_j, T_c)$ &   The occurred quantity of an item $i_j$ in $T_c$. \\ \hline
	
	$up(i_j)$ &   Each item $i_j \in I$ has a unit profit. \\ \hline
	
	$p(i_j, T_c)$ &  The profit of an item $i_j$ in $T_c$. \\ \hline
	
	$p(X, h)$ &   The sum of profits of $X$ in a period $h$. \\ \hline	
	
	\textit{RTWU} &  Redefined transaction-weighted utilization. \\ \hline
	
	\textit{pp(X)}  &  The sum of positive profit of $X$ in $D$. \\ \hline
	\textit{np(X)}  &  The sum of negative profit of $X$ in $D$. \\ \hline		
	\textit{rpp(X)}  &  The sum of remaining positive profit  of $X$ in $D$. \\ \hline	
	
	OPP-list  &  List structure with On-shelf Popularity and Profit. \\ \hline
	
	OFU$^{\pm}$-table &   An On-shelf Frequency-Utility (with both positive  \\ &  and negative profit value) table. \\ \hline
	
	\textit{X.list}  & The OPPP-list of a group of product $X$. \\ \hline

\end{tabular}
\end{table}

\begin{example}
Consider an e-commerce database shown in Table \ref{figDatabase}, which will be used as running example in the following sections. Similar to the e-commerce database provided by RecSys Challenge 2015\footnote{\url{https://recsys.acm.org/recsys15/challenge/}} (it contains some negative profit values since many all-occasion gifts are sold.), this example database contains five purchase behavior records ($T_1, T_2, \dots, T_5$) and three time periods ($1,2,3$). Behavior $T_1$ occurred in time period 2, and contains products $b$, $c$, and $e$, which respectively appear in $T_1$ with a purchase quantity of 2, 1 and 3. Table \ref{figUnitProfits} indicates that the external profit w.r.t. unit profit of these products are respectively -\$2, \$4 and \$7. Notice that the negative unit profit of a product $b$ indicates that this product is sold at a loss. 
\end{example}

\begin{table}
\centering
\caption{An e-commerce database}
\begin{tabular}{cclc} \hline
	\textbf{Tid}  & \textbf{User}	&  \textbf{Purchase record} & 	\textbf{Period}  \\ \hline
	$T_1$ &	 $U_1$  &  $(b,2)(c,1)(e,3)$ &	1 \\ 
	$T_2$ &	 $U_2$  &  	$(a,1)(b,1)(c,2)(f,1)$ &	1  \\
	$T_3$ &	 $U_3$  &  	$(a,3)(b,6)(c,4)(d,1),(e,1),(f,2)$ &	2 \\
	$T_4$ &	 $U_4$  &  	$(c,3)(d,3)(e,1)$ &	2 \\ 
	$T_5$ &	 $U_5$  &  	$(a,1)(d,2)(e,3)(f,1)$ &	3 \\ \hline
\end{tabular}
\label{figDatabase}
\end{table}

Given a time-varying \emph{e-commerce database} such that $D$ = \{$T_1$, $T_2$, $\dots$, $T_n$\} containing a set of temporal consumer purchase behaviors. Each transaction $T_c$ is a behavior record of one consumer, $T_c$ $\in D$ is a subset of $I$, and $T_c$ has a unique identifier $c$ called its \textit{Tid}. Let $I$ be a set of distinct products/items, $I$ = \{$i_1$, $i_2$, $\dots$, $i_m$\}. Each product/item $i_j \in I$ is associated with a positive or negative number $up(i_j)$, called its \emph{unit profit}. For each transaction $T_c$ such that $i_j$ $\in$ $T_c$, a positive number $q(i_j$, $T_c)$ is called \emph{purchase quantity} of $i$. Let \textit{PE} be a set of positive integers representing time periods, for any given period, this could be a weekly, monthly, quarterly or yearly timespan. Note that each transaction $T_c \in D$ is associated to a time period $pe(T_c) \in PE$, representing the duration time in a period, which the transaction occurred.

In general, the profit of $X \subseteq I$ is associated to the cost price and selling price. For the addressed problem in this paper, assume that the unit profit of each distinct product has been given in the pre-defined \textit{profit-table}, as shown in Table \ref{figUnitProfits}.  As mentioned before, it is usually seen in cross-promotion with negative profits.  Although giving away a unit of product $\{b\}$ results in a loss of \$4 for the supermarket, selling bundled products $\{(c,1) (e,3)\}$ that are cross-promoted with $\{b\}$ generates \$21 profit.

\begin{table}
\centering
\caption{External profit values (unit profit)}
\begin{tabular}{|c|c|c|c|c|c|c|} \hline
	\textbf{Product} & $a$ & $b$ & $c$ & $d$ & $e$ & $f$  \\ \hline
	\textbf{Profit} (\$) & 3 & -2 & 4 & 1 & 7 & 5 \\ \hline
\end{tabular}
\label{figUnitProfits}
\end{table}

Given a set of products $I$ and a set of customers $C$ = \{$c_1$, $c_2$, $\dots$, $c_j$\}, the market contribution of a group of products  $X$ = \{$i_1$, $i_2$, $\dots$, $i_j$\} is related to the profit from $X$ after marketing. Finally, the contribution of a group of products $X$ becomes the sum of the profits it receives from all the customers in the market. Therefore, the key concepts used in this paper are first introduced as follows. A utility function is a map $U$: $X \longrightarrow \Re$. The \emph{profit of a combined product} $X$ (a group of products $X \subseteq I$) in a transaction $T_c$ is:
\begin{equation} p(X, T_c) = \sum_{i_j \in X \wedge X \subseteq T_c} {p(i_j, T_c)}.  \end{equation} where $p(i_j, T_c)$ is the \emph{profit of a product} $i_j \in I$ in a transaction $T_c$, and $p(i_j, T_c)$ can be calculated as $p(i_j, T_c)$ = $up(i_j)$ $\times$ $q(i_j, T_c)$. It represents the profit generated by products $i \in X$ in $T_c$. Consider the set of time periods where $X$ was sold, the \emph{time periods (on-shelf time) of a group of products} $X \subseteq I$ becomes $os(X)$ = $\{pe(T_c) | T_c \in D$ $\wedge X$ $\subseteq T_c\}$. Let $p(X, h)$ denote the \emph{profit of a group of products} $X \subseteq I$ in a time period $h \in os(X)$, then by Eq. (2) we have: 
\begin{equation} p(X, h) = \sum_{T_c \in D \wedge h \in os(X) } {p(X, T_c)}.\end{equation}

By Eq. (3) we have the \emph{overall profit} of a group of products $X \subseteq I$ in an e-commerce database $D$ as $p(X)$ = $\sum_{h \in os(X)} {p(X,h)} $. Thus, given a group of products $X$, let $top(X)$ denote the \emph{total profit of the time periods about} $X$, then it can be represented in a function as follows: \begin{equation} top(X) = \sum_{h \in os(X) \wedge T_c \in D} {tp(T_c)}, \end{equation} where $tp(T_c)$ is the \emph{transaction profit} (\textit{tp}) of a transaction $T_c$, i.e., \begin{equation} tp(T_c) = \sum_{i \in T_c}{p(i, T_c)}.\end{equation} Let  $rp(X)$ denote the \emph{relative profit of a group of products} $X \subseteq I$  in an e-commerce database $D$, by Eq. (4) we have $rp(X)$ = $p(X)$/$top(X)$, and it represents the percentage of the profit that was generated by $X$ during the time periods where $X$ was sold.

\begin{example}
The profit of product $e$ in $T_1$ is $p(e, T_1)$ = 3 $\times $ \$7 = \$21, and the profit of products $\{c,e\}$ in $T_1$ is  $p(\{c,e\}, T_1)$ =  $p(c, T_1)$ + $p(e, T_1)$ = 1 $\times$ \$4 + 3 $\times $ \$7 = \$25. The time periods of $\{c,e\}$ are  $os(\{c,e\})$ = \{period 1, period2\}. The profit of $\{c,e\}$ in periods 1 and 2 are respectively $p(\{c,e\}$, \textit{period} 1) = \$25, and $p(\{c,e\}$, \textit{period} 2) = \$42. The profit of $\{c,e\}$ in the database is $p(\{c,e\})$ = $p(\{c,e\}$, \textit{period} 1) + $p(\{c,e\}$, \textit{period} 2) = \$25 + \$42 = \$67. The transaction profit of $T_1$,  $T_2$, $\dots$,  $T_5$ are $tp(T_1)$ = \$21, $tp(T_2)$ = \$14, $tp(T_3)$ = \$31, $tp(T_4)$ = \$20 and $tp(T_5)$ = \$31. The total profit of the time periods of $\{c,e\}$ is $top(\{c,e\})$ = $tp(T_1)$ + $tp(T_3)$ + $tp(T_4)$  = \$72. The relative profit of $\{c,e\}$ is $rp(\{c,e\}$ = $ p(\{c,e\})$/$top(\{c,e\})$ = \$67 / \$72 = 0.93.
\end{example}

\subsection{On-Shelf Availability}

On-Shelf Availability (OSA) \cite{corsten2003desperately} is a key factor of product for sale to the customer. It is impacted by a host of different factors, all along with the supply chain. Out of Stock (OOS) \cite{corsten2003desperately} is also known as stock-out, it is a situation where the retailer does not physically possess a particular product category, on its shelf, to sell this product to the customer. It can be estimated from store inventory data.

The retail industry being the highly competitive field, being able to fulfill customer expectations and the demands has become the most essential element in order to get sustainable growth and profit margin. Among them, on-shelf availability plays a key indicator for the retail industry, which can greatly impact the profit and customer loyalty. On the basis of the dataset shown in Table \ref{figDatabase}, and by taking the OSA, popularity and profit into account as the primary decision criteria, details of each product are shown in Table \ref{table:multiattributes}.

\begin{table}[htb]
\centering
\caption{Details of each product}
\label{table:multiattributes}
\begin{tabular}{cccll}
	\hline
	\multirow{2}*{\textbf{Product}}&
	\multirow{2}*{\textbf{Profit}}&
	\multirow{2}*{\textbf{Quantity}}
	&\multicolumn{2}{c}{\textbf{Seasonal period}}\\
	\cline{4-5}
	&&& \textbf{Start} &  \textbf{End} \\ \hline
	$a$ & \$12   & 4  &  1  & 3 \\
	$b$ & -\$18  & 9  &  1  & 2 \\
	$c$ & \$40   & 10 &  1  & 2 \\ 
	$d$ & \$6    & 6  &  2  & 3 \\ 
	$e$ & \$56   & 8  &  1  & 3 \\ 
	$f$ & \$20   & 4  &  1  & 3 \\ 
	\hline
\end{tabular}
\end{table}

Let $sup(h)$ denote the frequency of a time period $h$ w.r.t. $sup(h)$ = $|$the number of ${T_c \in h}|$, and $sup(X,h)$ denote the number of $|X \subseteq T_c \wedge T_c \in h|$, then the relative frequency of a group of products $X$ for a time period $h$ can be defined as:
\begin{equation} rf(X, h) = sup(X,h)/ sup(h). \end{equation}

Consider products $\{c\}$ and $\{a,c\}$ in period 1, $sup(\textit{period} 1)$ = 2, thus $rf(\{c, \textit{period} 1\})$ = $2/2$ = 1.0, and $rf(\{a,c\}, \textit{period} 1)$ = $1/2$ = 0.5. Besides, $rf(\{c, \textit{period} 1\}) >$  $rf(\{a,c\}, \textit{period} 1)$. The relative profit of a group of products $X$ for a time period $h$ is 
\begin{equation} rp(X, h) = p(X, h)/top(h), \end{equation} where $top(h)$ means the total profit of a time period $h$, and $top(h) =\sum_{T_c \in h \wedge T_c \in D} {tp(T_c)}$. Consider products $\{c\}$ and $\{a,c\}$ in period 1, $top(\textit{period} 1)$ = $tp(T_1)$ + $tp(T_2)$ = \$21 + \$14 = \$35, thus $rp(\{c, \textit{period} 1\})$ = \$12$/$\$35 = 0.343, and $rp(\{a,c\}, \textit{period} 1)$ = \$11$/$\$35 = 0.314.

\subsection{Problem Formulation}

A group of products $X$ is said to be the \emph{on-shelf most popular and profitable products} (OPPP) if it is popular in one or more periods $pt$ (its occurred relative frequency is no less than a user-specified minimum frequent threshold $minfre$), and it is high profitable in  $pt$ (its relative profit $rp(X)$ is no less than a user-specified minimum profit threshold $minpro$ given by the user $(0 \leq minpro \leq 1)$. Otherwise, $X$ is a \emph{non-OPPP}.

For clarity, when we use the terms of ``\textbf{OPPP}'', which indicates ``\underline{\textbf{O}}n-shelf most \underline{\textbf{P}}opular and \underline{\textbf{P}}rofitable \underline{\textbf{P}}roduct''. The OPPP makes maximal profit with various customer satisfactions in terms of on-shelf period, high popular and high profitable products. Thus, the significant concept of OPPP is actually a $k$-satisfiable product. If a product satisfies at least $k$-constraints (i.e., OSA constraint, popular w.r.t. inventory control, high profitable product), we say that this product is $k$-satisfiable product for maximizing profit where $k$ is a user's predefined parameter and a non-negative integer. In the addressed OP3M problem and the given running example, $k$ is equal to 3. To satisfy the customers, retailers must be able to obtain to their feedback and improve the services. Regarding to the above analytics, the decision maker can make the efficient business strategy and decision, which can improve the  overall profitability in his/her business. Ensuring the on shelf product is essential for any retailer, but even today it remains a major challenge.

\textbf{Problem statement.} \emph{The problem by computing the most popular on-shelf and profitable products for the targeted marketing} is to discover all significant OPPPs in an e-commerce database containing unit profit values are positive. The \emph{problem by computing the most popular on-shelf and profitable products for targeted marketing with negative values} is to discover all OPPPs in an e-commerce database where external unit profit values are positive or negative.

A naive way for this problem is to enumerate all possible subsets of products $I$, then calculate the sum of the frequency and profits of each possible subset, and choose the frequent subsets with the highly sum profit. However, this approach is not scalable because there is an exponential number of all possible subsets. This motivates us to propose an efficient algorithm named OP3M for the searching problem of OPPPs. For efficient multi-criteria decision analyses, as mentioned previously, we can utilize heuristic search \cite{guns2011declarative}, constraint programming \cite{guns2011itemset,guns2011k}, and multi-objective optimization  \cite{ugarte2017skypattern}.  An alternative approach is to formulate a truly multi-objective optimization problem where the heuristic search tries to optimize for each criterion with respect to each other. However, the state space of heuristic search methods is not fully explored and randomization is often employed. Therefore, OP3M utilizes exhaustive search strategy with various user-specified constraints instead of heuristic search. Overall, OP3M is an exact utility-based framework but not a randomize one.

\section{The Proposed OP3M Algorithm}
\label{sec:technic}

\subsection{Properties of On-Shelf Availability}

It can be demonstrated that the popularity measure is anti-monotonic \cite{agrawal1994fast,ugarte2017skypattern}, any superset of a non-popular pattern cannot be a popular pattern, while (relative) profit measure is not monotonic or anti-monotonic \cite{liu2005two,tseng2013efficient}. In other words, a product may have a lower, equal or higher profit than that of the profit of its subsets. We extend the concept of \emph{transaction-weighted utilization} (\textit{TWU}) \cite{liu2005two,ahmed2009efficient} in OPPP to show properties of on-shelf availability. For a given time period $h$, let $\textit{RTWU}(X, h)$ denote the \emph{redefined transaction-weighted utilization} of a group of products $X$ in $h$, thus it is the sum of the redefined transaction profit of transactions from $h$ containing $X$.

To handle the negative profit of product, the \emph{redefined transaction-weighted utilization} (\textit{RTWU}) \cite{lin2016fhn} of a group of products $X$ is defined as: 
\begin{equation}  \textit{RTWU}(X) = \sum_{T_c \in D \wedge X \subseteq T_c}{rtp(T_c)}.  \end{equation} where $rtp(T_c)$ is the \emph{redefined transaction profit} (abbreviated as \textit{rtp}) of a transaction $T_c$, that is $rtp(T_c)$  = $\sum_{x \in T_c \wedge p(x) >0}{p(X, T_c)}$. Thus, $rtp(T_c)$ contains the sum of the positive profit of the products in $T_c$, while negative external profits are ignored. Similarly, the \emph{redefined transaction-weighted utilization} of a group of products $X$ for a time period $h$ can be represented as: 

\begin{equation}  \textit{RTWU}(X, h) = \sum_{T_c \in D \wedge X \subseteq T_c \wedge pe(T_c) \subseteq h}{rtp(T_c)}. \end{equation}  

For the running example (considering that $b$ has an external profit value of $-\$2$), the \textit{rtp} of $T_1$, $T_2$, $T_3$, $T_4$, and $T_5$ are respectively \$25, \$16, \$43, \$20 and \$31. The \textit{RTWU} of products $a$, $b$, $c$, $d$, $e$, and $f$ are respectively \$90, \$84, \$104, \$94, \$119 and \$90. With reflexivity of OSA makes sense at all, the continuous utility function has several components, and the addressed problem becomes quite complicated. According to previous studies \cite{lin2016fhn,fournier2015foshu}, we generalize the properties of on-shelf availability for the case of an e-commerce database with time-sensitive periods as follows.

\begin{property}
\label{proptwu0}
\rm The \textit{RTWU} of a group of products $X$ for a period $h$ is an upper bound on the profit of $X$ in period $h$, that is $\textit{RTWU}(X, h)$ $\geq $ $p(X, h)$. 
\end{property}

\begin{property}
\label{proptwu2}
\rm The \textit{RTWU} measure is anti-monotonic in the whole database or in a specific period. Let $X$ and $Y$ be two products, if $X \subset Y$, then $\textit{RTWU}(X)$  $\geq $ $\textit{RTWU}(Y)$; for a period $ h $, it has $\textit{RTWU}(X,h)$ $\geq$ $\textit{RTWU}(Y, h)$.
\end{property}

\begin{property}
\label{twuprune}
\rm Let $X$ be a group of products, if $\textit{RTWU}(X)$/$top(X)$ $<$ \textit{minpro}, then the product $X$ is low profit as well as all its supersets.
\end{property}

\begin{property} 
\label{proptwu1}
\rm The \textit{RTWU} of a group of products $X$ divided by the total profit of its time periods $h$ is higher than or equal to its relative profit in periods $h$, i.e., $\textit{RTWU}(X, h)$ / $top(X)$ $\geq$ $rp(X, h)$. It is an upper bound on the relative profit of a group of products. 
\end{property}

\begin{property} 
\label{proptwu4}
\rm Given a group of products $X$, if there does not exist a time period $h$ such that $\textit{RTWU}(X, h)$/$top(X)$ $\geq$ \textit{minpro}, then $X$ is not an profitable on-shelf product. Otherwise, $X$ may or may not be an profitable on-shelf product. 
\end{property}

Consider products $\{c\}$ and $\{a,c\}$ in period 1, the values $\textit{RTWU}(\{c\})$ and $\textit{RTWU}(\{a,c\})$ are \$90 and \$59, which are the overestimations of $p(\{a\})$ = \$39 and $p(\{a,c\})$ = \$36, thus respecting Property \ref{proptwu1}. Suppose the above properties we've been studying, and the previous proposition true, we still cannot easily solve the OP3M  problem.

\subsection{Properties of Positive and Negative Profits}

Economic profit can be positive, negative, or zero. If your business generates a negative profit, this means that, for the time those products are sold less than their cost price. Thus, some products went less successful than others and got negative profit. Nevertheless, all profit together make a whole profit for each cosmetic product, which is common seen in the successful business. What if the profit target is negative and the final actual result is positive? The following properties hold.

First, let the total order $\prec$ on products in the designed algorithm adopts the \textit{RTWU} ascending order of products, and negative products always succeed all positive products. In the running example, the \textit{RTWU} of six products $a$, $b$, $c$, $d$, $e$, and $f$ are respectively as \$90, \$84, \$104, \$94, \$119, and \$90, thus the total order $\prec$ on products is $ a \prec$ $f \prec$ $d \prec c$ $\prec e \prec b $. The complete search space of the addressed problem can be represented by a Set-enumeration tree \cite{rymon1992search} where products are sorted according to the previous total order $\prec$. By utilizing the OPP-list and OFU$^{\pm}$-table, we named this tree as OPP-tree. In this OPP-tree, according to the total order $\prec$, all child nodes of any tree node are called its extension nodes.  For any products (product-sets) $X$, let $ pp(X) $ and $ np(X) $ respectively denote the sum of positive profits and the sum of negative profits of $X$ in a transaction or period or database, such that $ p(X)$ = $pp(X)$ + $np(X) $. We have the following important observations.

\begin{property}	
\rm Relationship between positive profits and negative profits of a group of products: $ np(X)$ $\leq p(X)$ $\leq pp(X) $ in a transaction or period or database \cite{lin2016fhn}. We respectively denote as $ np(X, T_c)$ $\leq p(X, T_c)$ $\leq$  $pp(X, T_c) $ in a transaction $T_c$, $ np(X, h)$  $\leq$ $p(X, h)$ $\leq$  $pp(X, h) $ in a period $h$, and $ np(X)$ $\leq$ $p(X)$ $\leq$ $pp(X) $ in a database.
\end{property}

It can be seen that the positive profit value of a group of products is always no less than its actual profit, while the negative profit value is just the opposite. Thus, the positive profit value is an upper-bound on the profit. However, both $pp(X)$ and $np(X)$ cannot be used to overestimate the profit of a group of products. As an upper-bound on profit,  $pp(X)$ still does not hold the \textit{downward closure} property for the extensions with positive or negative products.

\subsection{OPP-List and OFU$^{\pm}$-Table}

In this subsection, we introduce the new concept called ``list structure of a group of products with its On-shelf Popularity and Profit'' (OPP-list for short) which is a component used for the information storing and calculation. Besides, a new concept called \textit{remaining positive profit} is introduced and applied to obtain the estimated upper-bound, which will be presented in next subsection. First, the OPP-list structure is defined as follows. 

\begin{definition}
\rm Let $ rpp(X,T_c) $ denote the remaining positive profit of a group of products $X$ in a transaction $T_c$. Thus, $ rpp(X,T_c) $ is the sum of the positive profit values of each product appearing after $X$ in $T_c$ according to the total order $\prec$. It is represented as: 
\begin{equation}
rpp(X,T_c) = \sum_{ i_j \notin X \wedge X \subseteq T_c \wedge X \prec i_j }p(i_j,T_c), p(i_j,T_{c}) \geq 0.
\end{equation}
\end{definition}

\begin{definition}
\rm The OPP-list in an e-commerce database $D$ is a set of tuples corresponding to the transactions where $X$ appears. A tuple is defined as $<$$\underline{tid}$, $\underline{pp}$, $\underline{np}$, $\underline{rpp}$, \underline{\textit{period}} $>$ for each transaction $T_{c}$ containing $X$. 
\begin{itemize}
	\item $\underline{tid}$: the transaction identifier of $T_{c}$; 
	\item $\underline{pp}$: the positive profit of $X$ in $T_{c}$, i.e., $ p(X,T_{c}) \geq 0$; 
	\item $\underline{np}$: the negative profit of $X$ in $T_{c}$, i.e., $ p(X,T_{c}) < 0 $; 
	\item \underline{\textit{rpp}}: the remaining positive profit of $X$ in $T_{c}$, w.r.t. $ rpp(X,T_c) $;
	\item \underline{\textit{period}}: the related period where $T_{c}$ is occurred.
\end{itemize}
\end{definition}

\begin{figure}[!htbp]
\centering
\includegraphics[scale=0.5]{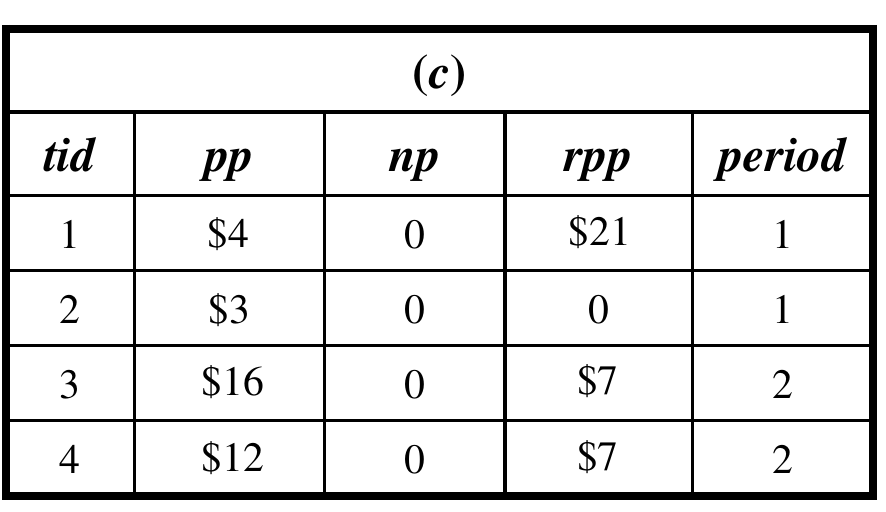}
\captionsetup{justification=centering}
\caption{Constructed OPP-lists of product ($c$).}
\label{fig:OPP-listOfC}
\end{figure}

\begin{example} 	
Since the total order $\prec$ on products is $ a \prec f \prec d \prec c \prec e \prec b $, we have that $ rpp(a,T_3)$ = $p(f,T_3)$ + $p(d,T_3)$ + $p(c,T_3) $ + $p(e,T_3) $ = \$10 + \$1 + \$16 + \$67 = \$34, and $ rpp(\{a,c\},T_{3})$ = $p(e,T_3)$ = \$7. Thus, the OPP-list of product $(c)$ is $\{(T_1$, \$4, 0, \$21, 1), $(T_2$, \$3, 0, 0, 1), $(T_3$, \$16, 0, \$7, 2), $(T_4$, \$12, 0, \$7, 2), $\}$, as shown in Fig. \ref{fig:OPP-listOfC}. It can perform a single database scan to create the all OPP-lists of all 1-products in the processed database. Based on the designed OPP-list and previous studies \cite{lin2016fhn,fournier2015foshu}, we can extract the following information.
\end{example}

For a group of products \textit{X}, let $pp(X, h)$, $np(X, h)$, and $rpp(X, h)$ are respectively the sum of \textit{pp} values, the sum of \textit{np} values and the sum of \textit{rpp} values in a period $h$ w.r.t. OPP-list of \textit{X}, that are:
\begin{equation}
pp(X, h) = \sum_{X \subseteq T_c \wedge T_c \in h}pp(X,T_c);
\end{equation}
\begin{equation}
np(X, h) = \sum_{X \subseteq T_c \wedge T_c \in h}np(X,T_c);
\end{equation}
\begin{equation}
rpp(X, h) = \sum_{X \subseteq T_c \wedge T_c \in h}rpp(X,T_c).
\end{equation}

Similarly, let $pp(X)$, $np(X)$, and $rpp(X)$ are respectively the sum of \textit{pp} values, the sum of \textit{np} values and the sum of \textit{rpp} values for a group of products \textit{X} in the database $D$ w.r.t. OPP-list of \textit{X}, we have:
\begin{equation}
pp(X) = \sum_{X \subseteq T_c \wedge T_c \in D}pp(X,T_c);
\end{equation}
\begin{equation}
np(X) = \sum_{X \subseteq T_c \wedge T_c \in D}np(X,T_c);
\end{equation}
\begin{equation}
rpp(X) = \sum_{X \subseteq T_c \wedge T_c \in D}rpp(X,T_c).
\end{equation}

By utilizing the property of OPP-list, we further design a structure called on-shelf frequency-profit table with both positive and negative value, hereafter termed OFU$^{\pm}$-table. The OFU$^{\pm}$-table of a pattern is built after the construction of its OPP-list, and it stores the following information.

\begin{definition}
\rm  An OFU$^{\pm}$-table of a group of products $ X $ contains nine parts: (1) \textit{\underline{name}}: the name of product $ X $; (2) \textit{\underline{sup(X, h)}}: the support of $ X $ in a period $h$; (3) \textit{\underline{sup(X)}}: the support of $ X $ in database $ D $; (4) \textit{\underline{pp(X, h)}}: the summation of the positive profits of $ X $ in a period $h$; (5)	\textit{\underline{pp(X)}}: the summation of the positive profits of $ X $ in database $ D $; (6) \textit{\underline{np(X, h)}}: the summation of the negative profits of $ X $ in a period $h$; (7) \textit{\underline{np(X)}}: the summation of the negative profits of $ X $ in database $ D $; (8) \textit{\underline{rpp(X, h)}}: the summation of the remaining positive profits of $ X $ in $h $; and (9) \textit{\underline{rpp(X)}}: the summation of the remaining positive profits of $ X $ in $ D $. Notice that there is a HashMap to respectively keep the \textit{\underline{sup(X, h)}}, \textit{\underline{pp(X, h)}}, \textit{\underline{np(X, h)}} and \textit{\underline{rpp(X, h)}} vaules of product $X$ in each period $h$.
\end{definition}

\begin{example} 
The construction process of a OFU$^{\pm}$-table is as follows. 
Consider a group of products $ (c) $ in Table \ref{figDatabase}, it appears in $T_1$, $T_2$, $T_3$, and $T_4$. From the built OPP-list of $ (c) $ which is shown in Fig. \ref{fig:OPP-listOfC}, the OFU$^{\pm}$-table of product $(c)$ is efficiently constructed using the support count, positive profit, negative profit and remaining positive profit. They are calculated during the construction of its OPP-list, and the results of its OFU$^{\pm}$-table are \{$sup(c)$ = 4, $ pp(c,1)$ = \$7, $ pp(c,2)$ = \$28, $ pp(c)$ = $ pp(c,1)$ + $ pp(c,2)$ = \$35, $ np(c,1)$ = $ np(c,2)$ = $ np(c)$ = \$0, and  $ rpp(c,1)$ = \$21, $rpp(c,2)$ = \$14, $ rpp(c)$ = $ rpp(c,1)$ + $ rpp(c,2)$ = \$35\}. 
\end{example}

After initially constructing the OPP-list and OFU$^{\pm}$-table of each 1-product/item, for any $k$-product ($k \geq 2 $), its' OPP-list can be directly calculated using the OPP-lists of some of its subsets, without scanning the database. The construction procedure of the OPP-list and OFU$^{\pm}$-table of a $k$-product, and is shown in Algorithm 1. Let be a  product $X$ and two products $X_a$ and $X_b$, they are extensions of $X$ by respectively adding two distinct products $a$ and $b$ to $X$. The construction procedure takes as input the OPP-lists of $X$, $X_a$ and $X_b$, and outputs the OPP-list and OFU$^{\pm}$-table of the pattern $ X_{ab} $. Specially, it is important to notice that the construction for $k$-product ($ k \geq 3 $, Lines 4 to 7) is different from that of $k$-product ($ k$ = 2, Line 9). For instance, the profit value of \{$a,b$\} is the sum profit value of \{$a$\} and \{$b$\}. And the 3-itemset \{$a,b,c$\} its OPP-list is constructed by the OPP-lists of \{$a,b$\} and \{$b,c$\}, and the sum profit value of \{$a,b$\} and \{$b,c$\} has a duplicate profit value of \{$b$\}. Thus, it should avoid duplication for $k$-product ($ k \geq 3 $). It can be easily implemented, because the all necessary information of ($k$-1)-product ($ k \geq 2 $) has been calculated before constructing the OPP-list of $(k)$-product ($ k \geq 2 $) w.r.t. the extension pattern. 


\begin{algorithm}
\LinesNumbered
\label{AlgorithmConstruct}
\caption{Construction procedure}
\KwIn{$ X $, $ X_{a} $, $ X_{b} $.}
\KwOut{ $ X_{ab}$.} 
set  $ X_{ab}.list \leftarrow \emptyset $, $ X_{ab}.table$ $\leftarrow \emptyset $\; 
\For {each tuple $ E_{a}\in X_{a}.list $}
{
	\If {$ \exists E_{a}\in$ $X_{b}.list$ $\wedge E_{a}.tid$ == $E_{b}.tid $}
	{
		\If{$ X.list \neq \emptyset $}
		{
			Search for element $ E\in X.list $, $E.tid$ = $E_{a}.tid $\;			
			$E_{ab} \leftarrow <E_{a}.tid$, $E_{a}.pp$ + $E_{b}.pp$ - $E.pp$, $E_{a}.np$ + $E_{b}.np$ - $E.np$, $E_{b}.rpp$, $E_{a}.\textit{period}>$\;
			
		}
		\Else
		{
			$E_{ab} \leftarrow <E_{a}.tid$, $E_{a}.pp$ + $E_{b}.pp$, $E_{a}.np$ + $E_{b}.np$, $E_{b}.rpp$, $E_{a}.\textit{period}>$\;
			
		}
		$ X_{ab}.list$ $\leftarrow$ $X_{ab}.list$ $\cup E_{ab} $\;
		update information in the OPP-table for $ X_{ab}$\;
	}	
}
\textbf{return} $ X_{ab} $\
\end{algorithm}


\subsection{Filtering Strategies for Searching}

\begin{lemma}
\label{theorem_GDC}
\rm (\textbf{Anti-monotonicity of the unpromising product with support}).
\rm In the search space w.r.t. OPP-tree, if a tree node is a popular product in the whole database $D$ or a period $h$, its parent node is also a popular product in $D$ or $h$. Let $ X $ be a $k$-products (node) and its parent node  are denoted as $ X'$, a ($k$-1)-products. For a given database $D$ or a period $h$, the \textit{relative frequency} measure is anti-monotonic: $ rf(X) \leq rf(X') $ always holds.
\end{lemma}

\begin{proof} 
According to well-known Apriori property \cite{agrawal1994fast}, it always exists the relationship $ sup(X,h)$ $\leq$ $sup(X',h) $. Thus, the \textit{downward closure} property of \textit{relative frequency} measure can be hold. 
\end{proof}

\begin{lemma} 
\label{lemma2}
(\textbf{Anti-monotonicity of unpromising product with profit upper-bound}). \rm For any node $X$ in the search space w.r.t. the OPP-tree, the sum of $\textit{SUM}(X.pp)$ and $\textit{SUM}(X.rpp)$ in the OPP-list of $X$ (within a period or the whole database) is larger than or equal to profit of any one of its children (within any period $h$ or the whole database $D$ w.r.t. the whole/maximal period in database).
\end{lemma}

Thus, there exists an upper bound on profit of any pattern/node with respect to a special period. Lemma \ref{lemma2} guarantees that the sum of profits of $ X $ in $D$ or $h$ w.r.t $p(X)$ is always less than or equal to the sum of \textit{SUM}$(X'.pp)$ and \textit{SUM}$(X'.rpp)$ in $D$ or $h$. It ensures that the \textit{downward closure} property of transitive extensions with positive or negative products, based on these observations, we can use the following two filtering strategies.

\begin{strategy}
When performing a depth-first search strategy on the OPP-tree, if the relative frequency of any product $X$  within a  time period $h$ is less than \textit{minfre} (w.r.t. $rf(X,h)$ = $sup(X,h)$/$sup(h)$), any of its child node is not an OPPP, they can be regarded as irrelevant and directly pruned.
\end{strategy}

\begin{strategy}
When traversing the OPP-tree based on a depth-first search strategy, if the sum of \textit{SUM}$(X.pp)$ and \textit{SUM}$(X.rpp)$ of any node/product $X$ within the related period of $X$ is less than \textit{minpro} (w.r.t. $rp(X)$ = $p(X)$/$top(X)$), any of its child node is not an OPPP, they can be regarded as irrelevant and be directly pruned.
\end{strategy}


\subsection{Main Procedure}

To clarify our methodology, we have illustrated the key properties of OSA and profit, the key data structures and the profit upper-bound so far. Utilizing the above technologies, as shown in Algorithm \ref{AlgorithmOP3}, the main procedure takes as input: (1) an e-commerce  database, $D$; (2) a user-specified profit-table, \textit{ptable}; (3) minimum frequent threshold, \textit{minfre}; and (4) a user-specified minimum profit threshold, \textit{minpro}. How to systematically select these thresholds? And how they could be optimized to ensure the performance of the proposed approach? Note that both \textit{minfre} and \textit{minpro} are user-specified based on user's priori knowledge and empiricism. In other wolds, when applying the OP3M algorithm for mining OPPPs in different databases, the parameter setting is different. For example, a pattern may be on-shelf popular and high-profitable in one database, while it may be not in another one. Besides, the number of time periods in a database is inherent characteristic. We can manually set it according to domain knowledge and empiricism, such as week, month, quarter, or year. Therefore, \textit{minfre}, \textit{minpro}, and time periods can be manually specified case-by-case. It is worth mentioning that various parameter settings may lead to different performance. It is hard to optimize them to ensure the performance of the OP3M algorithm in all databases, but them can be tuned to achieve optimal performance in a special database.

The OP3M algorithm first scans the database to calculate $\textit{RTWU}(\{i\})$, $\textit{RTWU}(\{i\},h)$ and $os(\{i\})$ for each product $i$. Moreover, the set of all time periods $PE$ and the profit $top(h)$ of each period $h$ is computed during the first database scan. Notice that for optimization, the set of time periods $os(\{i\})$ of each product $i$ is represented as a bitset where the $k$th-bit is set to 1 if $i$ appear in the period $k$, otherwise 0. The bitset representation can quickly calculate the time periods of any product $X$ = $\{x_1$, $x_2$, $\dots$, $x_n\} $ by using the logical AND operation ($\oplus$), i.e., $os(X)$ = $os(x_1) \oplus os(x_2)$ $\oplus \dots os(x_n)$. 

Then, the algorithm computes for each product $i$ the value  $top(\{i\})$ using $os(i)$ and the profit of periods previously obtained. This allows us to create the set $I^*$ containing all products $i$ such that \textit{RTWU}($\{i\},h$) / $top(\{i\})$ $\geq $ \textit{minpro}. Thereafter, all products not in $I^*$ will be ignored since they cannot be part of OPPPs. The \textit{RTWU} values of products are then used to establish a total order $\succ$ on products, which is the ascending \textit{RTWU} order. A second database scan is then performed and the products in transactions are reordered according to the total order $\succ$; the OPP-list of each product $i \in I^*$ is built. After the construction of the OPP-list, the depth-first search exploration of products starts calling the recursive procedure \textit{Search} with the empty product $\emptyset$, the set of single products $I^*$, \textit{minfre}, \textit{minpro}, and the set of all time periods $PE$.

\begin{algorithm}[h]
\LinesNumbered
\SetKwFunction{Search}{Search}
\SetKwInOut{Input}{input}\SetKwInOut{Output}{output}
\Input{\textit{D}; \textit{ptable}; \textit{minfre}; \textit{minpro}}
\Output{OPPPs} 

\BlankLine
Scan $D$ and \textit{ptable} to calculate \textit{RTWU}($\{i\})$, \textit{RTWU}($\{i\},h)$ and $os(\{i\})$ for each product $i$, and calculate $top(h)$ for each period $h$, as well as the set of all time periods $PE$\;
Find $I^* \leftarrow$ each product $i \in I$ such that \textit{RTWU}($\{i\},h$)/$top(\{i\})$ $ \geq$ \textit{minpro}\;
Sort $I^*$ using the \textit{RTWU} ascending order as the total order $\succ$\; 
Scan $D$ to build the OPP-list of each product $i \in I^*$\;
\Search($\emptyset$, $I^*$, \textit{minfre}, \textit{minpro}, $PE$)\;
\textbf{return} \textit{OPPPs}\
\caption{The OP3M algorithm}
\label{AlgorithmOP3}
\end{algorithm}

The \textit{Search} procedure (Algorithm \ref{AlgorithmSEARCH}) takes as input: (1) a group of products $P$, (2) extensions of $P$ having the form $Pz$ meaning that $Pz$ was previously obtained by appending a product $z$ to $P$, (3) \textit{minfre}, (4) \textit{minpro}, and (5) the time periods of $P$ ($os(P)$). The search procedure operates as follows. For each extension $Px$ of $P$, if the related frequency of $Px$ is no less than \textit{minfre}, and the sum of the actual related profits values of $Px$ in the OPP-list is no less than \textit{minpro}, then $Px$ is output as a OPPP (Lines 2 to 6). Then, it uses the pruning strategies to determine whether the extensions of $Px$ would be the OPPPs and should be explored (Line 7). This is performed by merging $Px$ with all extensions $Py$ of $P$ such that $y \succ x$, $rf(Px,h)$ $\geq $ \textit{minfre} and \textit{RTWU}($\{x,y\},h$) $\geq $ \textit{minpro} (Line 10), to form extensions of the form $Pxy$ containing $|Px|$+1 products. The OPP-list of \textit{Pxy} is then constructed by calling the \textit{Construct} procedure to join the OPP-lists of $P$, $Px$ and $Py$ (Lines 10 to 15). Only the promising OPP-lists would be explored in next extension (Line 14). Then, a recursive call to the \textit{Search} procedure with $Pxy$ is performed to calculate its on-shelf popularity and profit and explore its extension(s) (Line 17).

\begin{algorithm}
\caption{The \textit{Search} procedure}
\label{AlgorithmSEARCH}
\LinesNumbered
\SetKwFunction{Construct}{Construct}\SetKwFunction{Search}{Search}\SetKwFunction{Construct}{Construct}
\SetKwInOut{Input}{input}\SetKwInOut{Output}{output}
\Input{$P$: a group of products,  \textit{ExtenOfP}: a set of extensions of $P$, \textit{minfre}, \textit{minpro}, $os(P)$: the time periods of $P$}
\Output{the set of OPPPs}
\BlankLine
\ForEach{ product $Px \in $ \textit{ExtenOfP}} {
	$os(Px) \leftarrow os(P) \cap os(x) $\;
	Calculate $top(Px)$\;
	\If{$ \exists h \in os(Px)$ such that $rf(Px,h)$ $\geq $ \textit{minfre} $\wedge$  $\textit{SUM}(Px.pp)$ + $\textit{SUM}(Px.np)$/$top(Px)$ $\geq$ \textit{minpro}}{
		output $Px$\;
	}
	
	\If{$ \exists h \in os(Px)$ such that $rf(Px,h)$ $\geq$ minfre $\wedge$ $\textit{SUM}(Px.pp)$ + $\textit{SUM}(Px.rpp)$/$top(Px)$ $\geq$ \textit{minpro}}{
		\textit{ExtenOfPx} $\leftarrow \emptyset$\;
		\ForEach{ product $Py \in$ \textit{ExtenOfP} such that $y \succ x$} {
			$Pxy \leftarrow  Px \cup Py$\;
			\If{$ \exists h \in os(Pxy)$ such that  $\textit{RTWU}(\{x,y\},h)$ / $top(Pxy) \geq \textit{minpro}$}{
				$Pxy.list \leftarrow $ \Construct$(P, Px, Py$)\;
				\textit{ExtenOfPx} $\leftarrow$ \textit{ExtenOfPx} $\cup$ \textit{Pxy.list}\;
			}
		}
		call \Search($Px$, \textit{ExtenOfPx}, \textit{minfre}, \textit{minpro}, $os(Px)$)\;
	}
}

\textbf{return} \textit{OPPPs}\
\end{algorithm}

Complexity analysis. Assume $I$ = $\{i_1,$ $i_2,$ $\dots, i_m\}$ be a finite set of $m$ distinct items, and $D$ = \{$T_1$, $T_2$, $ \dots, T_n$\}. Firstly, a single database scan performs in $O(nz)$ time, where $z$ is the average transaction length. In the worst case, it takes $O(nm)$ time. The construction of OPP-list and OFU$^{\pm}$-table is done in linear time. An exhaustive search of the search space in OP3M would take $O(2^{m} - 1)$ time. However, in real situation, database may be rarely very sparse or very dense. Thus, the number of items in the longest products in OP3M is generally much less than $m$, and search space is $2^{m} - 1$ (the complete number of itemsets in the search space), in the worst case.  The space analysis is as follow. The number of initial OPP-list and OFU$^{\pm}$-table is $|I^*|$, and each OPP-list takes $O(n)$ space if it contains an entry for each transaction. Firstly, the total space required for building the initial OPP-lists of 1-products is in the worst case $O(|I^*| \times n)$ space. Therefore, the worst-case space complexity of OP3M is $O((2^{m} - 1) \times n)$. Incorporating the constraints of frequency, time periods, and profit, the filtering strategies above can lead to a smaller search space than the worst case. In practice, the real search space of the OP3M algorithm is reasonable and it has a linear time, which will be shown in experimental evaluation.

\section{Experimental Study} 
\label{sec:experiments}

In this section, we study the proposed algorithm on several real datasets to evaluate its effectiveness and efficiency. To the best of our knowledge, this is the first work to address the targeted marketing problem for finding the on-shelf popular and most profitable products by considering product frequency, purchase time periodic, on-shelf availability and utility theory. Thus, none existing methods in literature can be reasonably compared here, as the baseline, to evaluate the efficiency (w.r.t. execution time, memory usage, etc.) against the proposed model. 

To analyze the usefulness of the OP3M framework, the derived frequent patterns (FPs, generated by the well-known FP-growth algorithm \cite{han2004mining}), high utility/profitable patterns (HUPs, generated by the FHN algorithm \cite{lin2016fhn}), and OPPPs (generated by the proposed OP3M algorithm) on the same datasets are examined. Thus, in Section 5.2, the well-known FP-growth, UP-growth are conducted as the baseline algorithms. A comparison with the frequency-based FP-growth approach will demonstrate whether the utility-based method is superior. Specially, we also compare the total utility of the derived patterns (e.g., HUPs and OPPPs), to evaluate how can the proposed OP3M algorithm towards revenue maximization. It is worth mentioning that OP3M, as an exact utility-based framework, utilizes exhaustive search strategy with constraints but not heuristic search. Therefore, some methods algorithm with existing heuristics are not compared here.

\subsection{Datasets and Data Preprocessing}

All compared algorithms are implemented using the Java language and executed on a PC ThinkPad T470p with an Intel Core i7-7700HQ CPU @2.80GHz and 32 GB of memory, run on the 64 bit Microsoft Windows 10 platform. Typically e-commerce datasets are proprietary and consequently hard to find among publicly available data. We support reproducibility and use four publicly available  e-commerce datasets\footnote{\url{http://www.philippe-fournier-viger.com/spmf/}} (mushroom, chess, retail, and kosarak) in our experiments.  These datasets have varied characteristics and represented the main types of data typically encountered in real-life scenarios. The characteristics of used datasets are described below in details. 

$ \bullet $ \textit{\textbf{mushroom}}: a very dense dataset containing 8124 transactions with 119 distinct items. Its average item count per transaction is 23, with a density ratio as 19.33\%.

$ \bullet $  \emph{\textbf{chess}}: it contains 3,196 transactions with 75 distinct products and an average transaction length of 36 products. It is a very dense dataset, with a density ratio as 49.33\%.

$ \bullet $ \emph{\textbf{retail}}: it is a sparse e-commerce dataset, which contains 88,162 purchase records with 16,470  distinct products and an average transaction length of 10.30 products.

$ \bullet $ \textit{\textbf{kosarak}}: a very large dataset containing 990,002 transactions of click-stream data from a hungarian on-line news portal, it has 41,270 distinct products.

\subsection{Effectiveness Analytics}

The addressed OP3M problem aims at computing the $k$-satisfiable on-shelf most popular and profitable products. Thereby, this OPPP explicitly includes on-shelf availability, the frequency, and profit of contributions. How to know whether the results is interpreted or not? Numerous studies have shown that two metrics, frequency and utility, are good estimations of the importance of a pattern. And those frequency-based or utility-based methods have broad applications (see \cite{geng2006interestingness,gan2018survey,gan2020survey} for an overview). Therefore, it is make sense to evaluate the number, frequency, and profit of the mining results of OP3M. Results of different kinds of generated patterns under various parameters are shown in Table \ref{table:patterns1}. 

Notice that mushroom is first tested with a fixed  \textit{minfre}: 6\% and various \textit{minpro} from 5\% to 13\%, and then tested with a fixed \textit{minpro}: 5\% and various \textit{minfre} from 6\% to 14\%. Chess is tested in a similar way. For the retail dataset, it is performed under a fixed  \textit{minfre}: 0.07\% and various \textit{minpro} from 0.20\% to 0.28\%, then performed under a fixed  \textit{minpro}: 0.20\% and various \textit{minfre} from 0.05\% to 0.13\%. Besides, We ran OP3M algorithm on the same datasets but randomly grouped transactions into 5, 25 and 50 time periods, notice that \#${P5}$, \#${P25}$ and \#${P50}$ are the number of OPPPs which is respectively derived by the same OP3M algorithm on the dataset with different number of time periods 5, 25 and 50.

\begin{table}[htb]
\fontsize{6.5pt}{9pt}\selectfont
\centering
\caption{Derived patterns under various parameters}
\label{table:patterns1}
\begin{tabular}{|c|c|lllll|}
	\hline\hline
	\multirow{2}*{\textbf{Dataset}}&
	\multirow{2}*{\textbf{Pattern}}
	&\multicolumn{5}{c|}{\textbf{Results with varied threshold (\textit{minfre} or \textit{minpro})}}\\
	\cline{3-7}
	&& \textbf{test$_1$}  & \textbf{test$_2$}  &  \textbf{test$_3$}   & \textbf{test$_4$}  &   \textbf{test$_5$}   \\ \hline

	&  FPs & 1,843,327   & 1,843,327  &  1,843,327   &  1,843,327	&  1,843,327  \\
	&  HUPs &  -  &  - &   -  &   -	&  -   \\
	\textbf{mushroom}  &  \textbf{\#${P5}$} & 222,251 & 154,513 & 102,642 & 66,917	&  42,717	 \\
	(fix \textit{minfre}: 6\%) &\textbf{\#${P25}$} & 222,251	& 154,513	& 102,642	 &   66,917 	&  42,717 	\\
	&  \textbf{\#${P50}$} & 222,251 & 154,513 & 102,644 & 66,917	& 42,717	 \\
	\hline

	&  FPs &  1,843,327  &  650,003  &  600,817   &   150,137	&   104,629  \\
	&  HUPs &  -  & -  &  -   &   -	&   -  \\
	\textbf{mushroom} &  \textbf{\#${P5}$} & 222,251	&  102,452	& 94,882	& 30,659	& 25,579	\\
	(fix \textit{minpro}: 5\%) &\textbf{\#${P25}$} & 222,251	&  102,452	& 94,882	& 30,659	& 25,579	\\
	&  \textbf{\#${P50}$} &  222,251	&  102,452	& 94,882	& 30,659	& 25,579	\\
	\hline

	&  FPs & 1,272,932  & 1,272,932 &  1,272,932  & 1,272,932	&  1,272,932 \\
	&  HUPs & 32,324  & 14,114 &  5,847  & 2,304	&  859 \\
	\textbf{chess}  &  \textbf{\#${P5}$} & 16,596 & 8,639 & 4,146 & 1,848	& 761	 \\
	(fix \textit{minfre}: 50\%) &\textbf{\#${P25}$} & 16,596	& 8,639	& 4,146	& 1,848	& 761	\\
	&  \textbf{\#${P50}$} & 16,596 & 8,639 & 4,146 & 1,848	& 761	 \\
	\hline
	
	&  FPs & 2,832,777  & 1,272,932 &  574,998  & 254,944	&  111,239 \\
	&  HUPs & 14,114  & 14,114 &  14,114  & 14,114	&  14,114 \\
	\textbf{chess} &  \textbf{\#${P5}$} & 11,335 & 8,639 & 5,584 & 3,097	& 1,413	 \\
	(fix \textit{minpro}: 40\%) &\textbf{\#${P25}$} & 11,335	& 8,639	& 5,584	& 3,097	& 1,413	\\
	&  \textbf{\#${P50}$} & 11,335 & 8,639 & 5,584 & 3,097	& 1,413	 \\
	\hline

	&  FPs & 12,418  & 12,418 &  12,418  & 12,418	&  12,418 \\
	&  HUPs & 10,524  & 9,103 &  7,872  & 6,965	&  6,194 \\
	\textbf{retail} &  \textbf{\#${P5}$} & 5,092 & 4,850 & 4,619 & 4,384	& 4,159	 \\
	(fix \textit{minfre}: 0.07\%) &\textbf{\#${P25}$} & 13,696	& 12,943	& 12,149	& 11,437	& 10,675	\\
	&  \textbf{\#${P50}$} & 8,590,024 & 8,501,157 & 8,394,752 & 8,287,240	& 8,187,650	 \\
	\hline
	
	&  FPs & 19,242  & 12,418 &  8,829  & 6,749	&  5,282 \\
	&  HUPs & 10,524  & 10,524 &  10,524  & 10,524	&  10,524 \\
	\textbf{retail} &  \textbf{\#${P5}$} & 6,857 & 5,092 & 3,901 & 3,116	& 2,515	 \\
	(fix \textit{minpro}: 0.20\%) &\textbf{\#${P25}$} & 13,477,640	& 13,696	& 4,199	& 3,222	& 2,555	\\
	&  \textbf{\#${P50}$} & 73,1663,259 & 8,590,024 & 7,907,427 & 7,173,798	& 3,361	 \\
	\hline			
	
	\hline
	\hline
\end{tabular}
\end{table}

It can be clearly observed that the number of OPPPs is always different from that of FPs and HUPs under various \textit{minfre} and \textit{minpro} thresholds. Specifically, the \textit{minfre} and \textit{minpro} thresholds and the number of time periods all influence the results of OPPPs. However, the FPs is only influenced by \textit{minfre}, and the HUPs is only influenced by \textit{minpro}. For example, when setting \textit{minfre}: 0.07\% and  \textit{minpro}: 0.20\% on retail, the number of HUPs, \#${P5}$ (with 5 periods), \#${P25}$ (with 25 periods), and \#${P50}$ (with 50 periods), are respectively as  10,524, 5,092, 13,696 and 8,590,024, while there are 12,418 FPs. These patterns (HUPs and OPPPs with 5, 25 and 50 periods) respectively have the total utility as 3.76E7, 2.70E7, 2.97E7 and 1.40E9, but these results are not contained in Table \ref{table:patterns1} due to the space limit. When setting \textit{minfre}: 0.07\% and \textit{minpro}: 0.28\% on retail, the number of HUPs, \#${P5}$, \#${P25}$ and \#${P50}$ are changed to 6,194, 4,159, 10,675 and 8,187,650, and their total utility are 3.05E7, 2.55E7, 2.77E7 and 1.38E9, respectively. The reason is that a huge number of frequent patterns are always found, while few of them are on-shelf popular with high profit. Moreover, it is clear that both the period and popular factors affect the derived results of the addressed problem in terms of the number of patterns and the  achieved total utility. And larger the granularity of on-shelf period is, the higher revenue maximization can be achieved.

What is more, it is interesting to notice that the number of OPPPs using the developed OP3M framework may increase when the number of time periods in the processed dataset increases. This can be clearly observed from the results of \#${P5}$, \#${P25}$ and \#${P50}$, as shown in retail. Besides, the time period does not affect the very dense dataset since each transaction has the similar products and quantity, the on-shelf hot and profitable patterns within a short time period is likely to be an OPPP within a long time period. The distribution of $k$-patterns of the derived \#${P5}$, \#${P25}$ and \#${P50}$ is skipped here due to space limitation. In general, the frequent patterns do not usually contain a large portion of the desired profitable on-shelf patterns, the complete information (e.g., \textit{OSA}, profit) may be ignored. This implies the importance of inferring and understanding consumers' adoption behavior. In particular, using methods devised from information search and economics utility theory, we can focus on understanding the economic behavior of users from the periodic behavior in the historical data.

\subsection{Efficiency Analytics}

From Table \ref{table:patterns1}, we can observe that the influence of \textit{minfre} threshold, \textit{minpro} threshold, and the number of time periods. Furthermore, we performed an experiment to assess the influence of the number of time periods on the execution time with the same parameter setting at Table \ref{table:patterns1}. Notice that OP3M$ _{P5}$, OP3M$ _{P25}$ and OP3M$ _{P50}$ are respectively the running time of OP3M algorithm on the dataset with different number of time period 5, 25 and 50.  Results are shown in Fig. \ref{fig:Runtime} for the four  datasets. As we can see, the designed OP3M has much better scalability w.r.t to the number of periods on all datasets under various parameters. For example, when varying \textit{minrpro} on chess, as shown in Fig. \ref{fig:Runtime}(b), it always has OP3M$_{P5} <$ OP3M$_{P25} <$ OP3M$_{P50}$, the same trend can also be observed in Fig. \ref{fig:Runtime}(a) and Fig. \ref{fig:Runtime}(c). The reason why OP3M with less period performs better is that it mines all time periods at the same time; less period makes it earlier to achieve the conditions of pruning strategies in OP3M, thus leading to less computation time. OP3M mines them separately and merge results found in each-time periods, which degrades its performance when the number of time periods is large.

\begin{figure*}[!htbp]
\centering 
\includegraphics[scale=0.5]{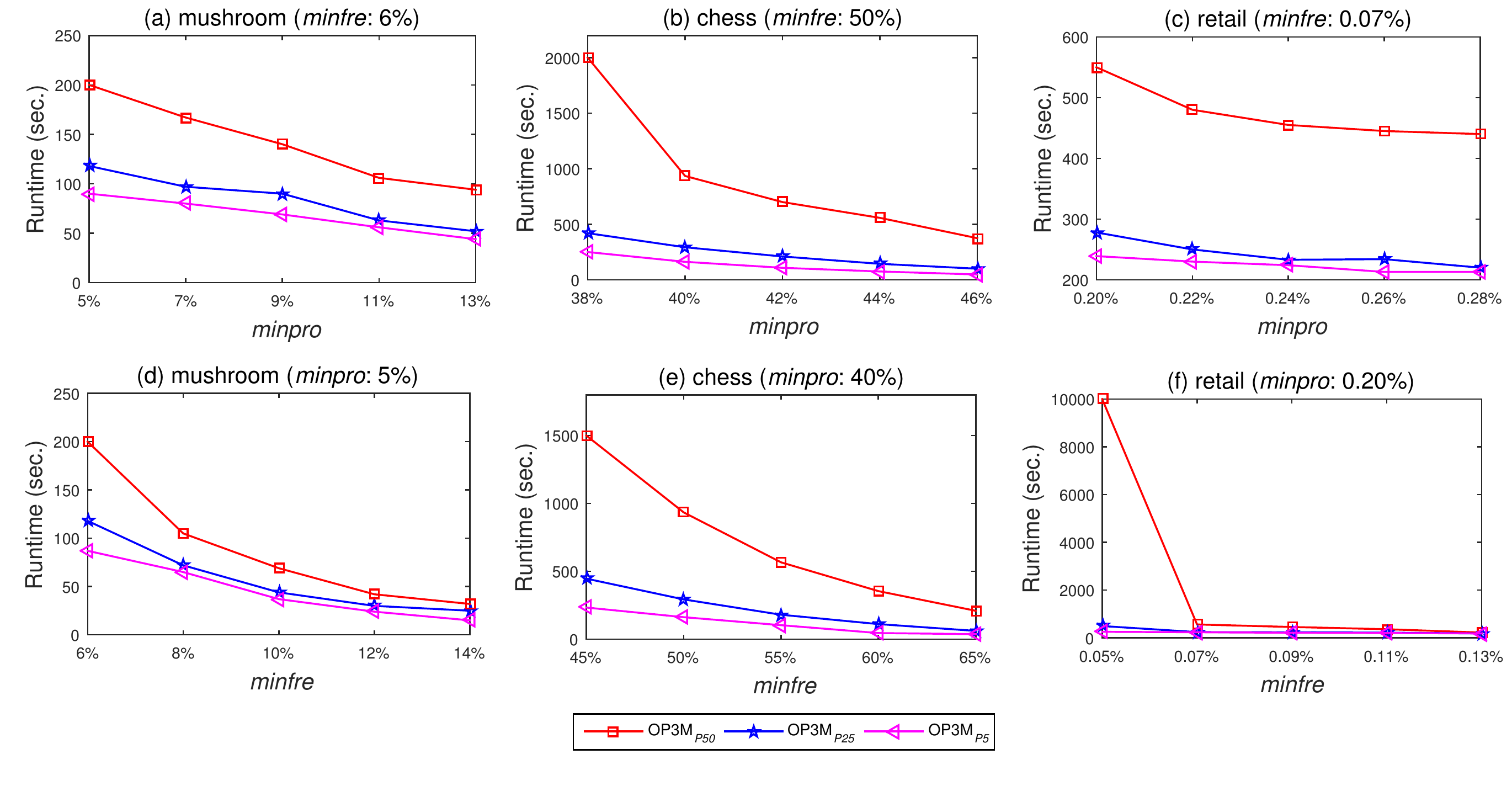}
\captionsetup{justification=centering}
\caption{Runtime under various parameters (\textit{minfre},  \textit{minpro} and period).}
\label{fig:Runtime}	
\end{figure*}

\begin{figure}[!htbp]
\centering
\includegraphics[trim=100 160 300 0,clip,scale=0.5]{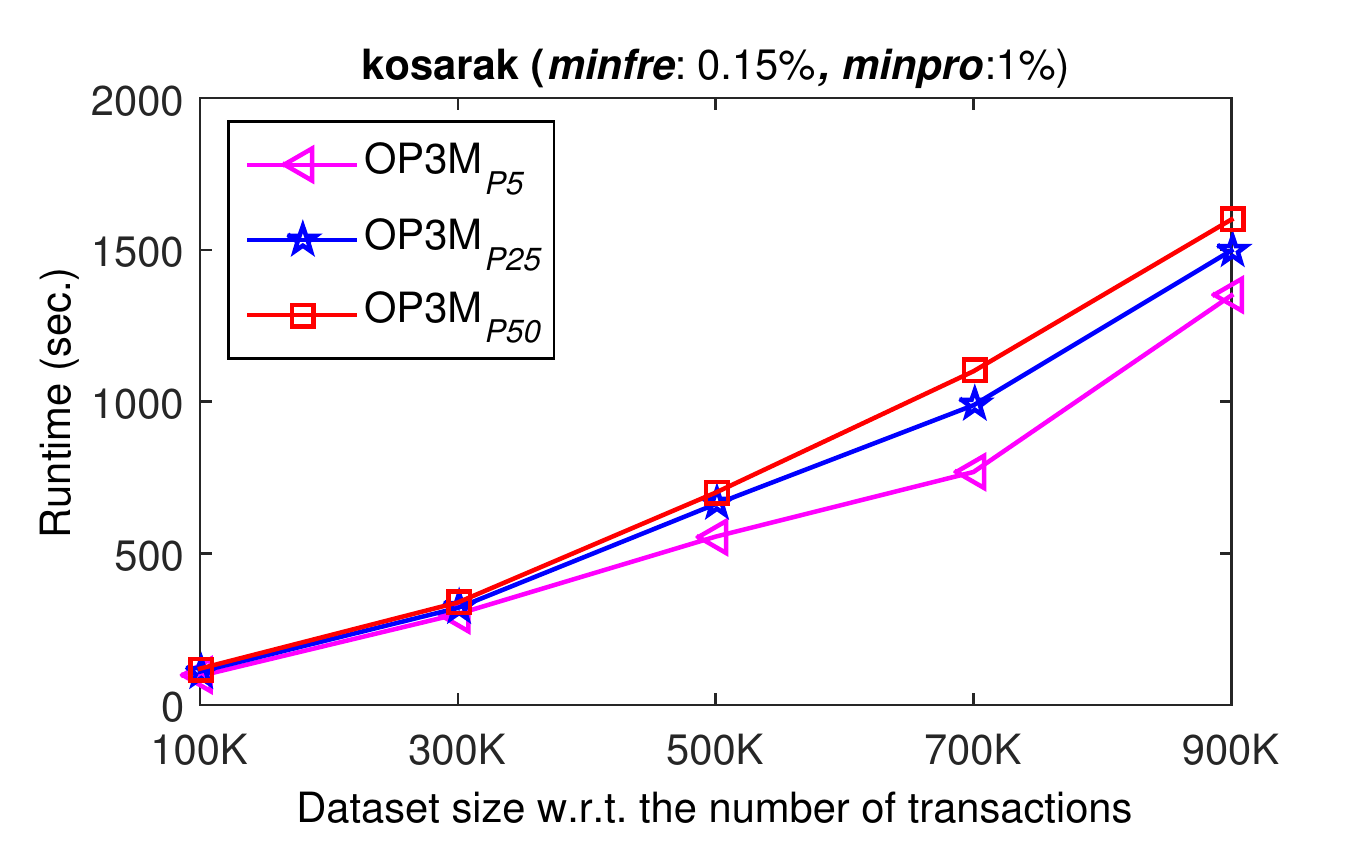}
\caption{Scalability test.}
\label{fig:scalability}
\end{figure}

Furthermore, we performed an experiment to asses the scalability of OP3M w.r.t the number of transactions. We ran the proposed profit-based OP3M model on kosarak dataset with \textit{minfre}: 0.15\% and \textit{minpro}: 1\%, and varied the number of transactions from 100,000 to 900,000. Results of scalability test are shown in Fig. \ref{fig:scalability}. In general, runtime is an estimate of how long it takes to perform such an analysis. Form the results, it can be observed that OP3M has linear scalability w.r.t the number of transactions.

\section{Conclusion} \label{sec:conclusion}

In this paper, we have presented a novel framework named OP3M for searching on-shelf popular and most profitable products in databases where both positive and negative profit values appeared and on-shelf time of products are considered. This is the first work to systematically study the problem of profit-based optimization for Economic behavior, including purchase frequency, on-shelf availability w.r.t. purchase time periodic, and profit. Given some historical datasets on market share, the designed algorithm can help us to make sense of the users' economic behavior and find the on-shelf popular and most profitable products. OP3M also brings several improvements over existing technologies. Based on the developed profit-based OPP-lists, it is a single phase algorithm that does not need to maintain candidates in memory. It relies on the novel concept named \textit{remaining positive profit}, and uses a depth-first search rather than a level-wise search. Moreover, OP3M finds OPPPs in all time periods at the same time rather than separately searching each period and performing costly intersection operations of the results of each time period. The extensive performance evaluation on several datasets demonstrates the effectiveness and efficiency of the OP3M algorithm.


\bibliographystyle{ACM-Reference-Format}
\bibliography{main}


\end{document}